\documentclass{article}
\usepackage{ijcai21}

\usepackage{times}
\usepackage{soul}
\usepackage{url}
\usepackage[hidelinks]{hyperref}
\usepackage[utf8]{inputenc}
\usepackage[small]{caption}
\usepackage{graphicx}
\usepackage{amsmath}
\usepackage{amsthm}
\usepackage{booktabs}
\usepackage{algorithm}
\usepackage{algorithmic}
\usepackage{subcaption}
\usepackage{array}
\newtheorem{theorem}{Theorem}

\title{Federated Learning with Fair Averaging}

\author{

    Zheng Wang$^1$\and
    Xiaoliang Fan$^{1,}$\footnote{Corresponding Author}\and 
    Jianzhong Qi$^2$\and\\
    Chenglu Wen$^1$\and
    Cheng Wang $^1$\And
    Rongshan Yu $^1$
    \affiliations
    $^{1}$Fujian Key Laboratory of Sensing and Computing for Smart Cities, School of Informatics, Xiamen University, Xiamen, China\\
    $^2$School of Computing and Information Systems, University of Melbourne, Melbourne, Australia
    \emails
    zwang@stu.xmu.edu.cn,\\
    fanxiaoliang@xmu.edu.cn,
    jianzhong.qi@unimelb.edu.au,
    \{clwen, cwang, rsyu\}@xmu.edu.cn
    }

\begin{document}

\maketitle

\begin{abstract}
    Fairness has emerged as a critical problem in federated learning (FL). In this work, we identify a cause of unfairness in FL -- \emph{conflicting} gradients with large differences in the magnitudes. To address this issue, we propose the federated fair averaging (FedFV) algorithm to mitigate potential conflicts among clients before averaging their gradients. We first use the cosine similarity to detect gradient conflicts, and then iteratively eliminate such conflicts by modifying both the direction and the magnitude of the gradients. We further show the theoretical foundation of FedFV to mitigate the issue conflicting gradients and converge to Pareto stationary solutions. Extensive  experiments on a suite of federated datasets confirm that FedFV compares favorably against state-of-the-art methods in terms of fairness, accuracy and efficiency. The source code is available at \url{https://github.com/WwZzz/easyFL}.
\end{abstract}

\section{Introduction}
    Federated learning (FL) has emerged as a new machine learning paradigm that aims to utilize clients' data to collaboratively train a global model while preserving data privacy \cite{fedavg}. The global model is expected to perform better than any locally trained model, since it has much more data available for the training. However, it is difficult for the global model to treat each client fairly \cite{advance,qfedavg,afl}. For example, a global model for face recognition may work well for younger users (clients) but may suffer when being used by more senior users, as the younger generation may use mobile devices more frequently and contribute more training data. 
    
     Techniques have been proposed to address the fairness issue in FL. AFL \cite{afl} aims at good-intent fairness, which is to protect the worst-case performance on any client. This technique only works on a small network of dozens of clients because it directly treats each client as a domain, which may suffer in generalizability. Li \emph{et al.} \shortcite{qfedavg} seek to balance the overall performance and fairness using a fair resource allocation method. Hu \emph{et al.} \shortcite{fedmgda} find a common degrade direction for all clients without sacrificing anyone's benefit and demonstrate robustness to inflating loss attacks. Li \emph{et al.} \shortcite{fedmultitaskcc} further explore a trade-off between a more general robustness and fairness, and they personalize each client's model differently.  
    \begin{figure}
    \centering
    \includegraphics[scale=0.28]{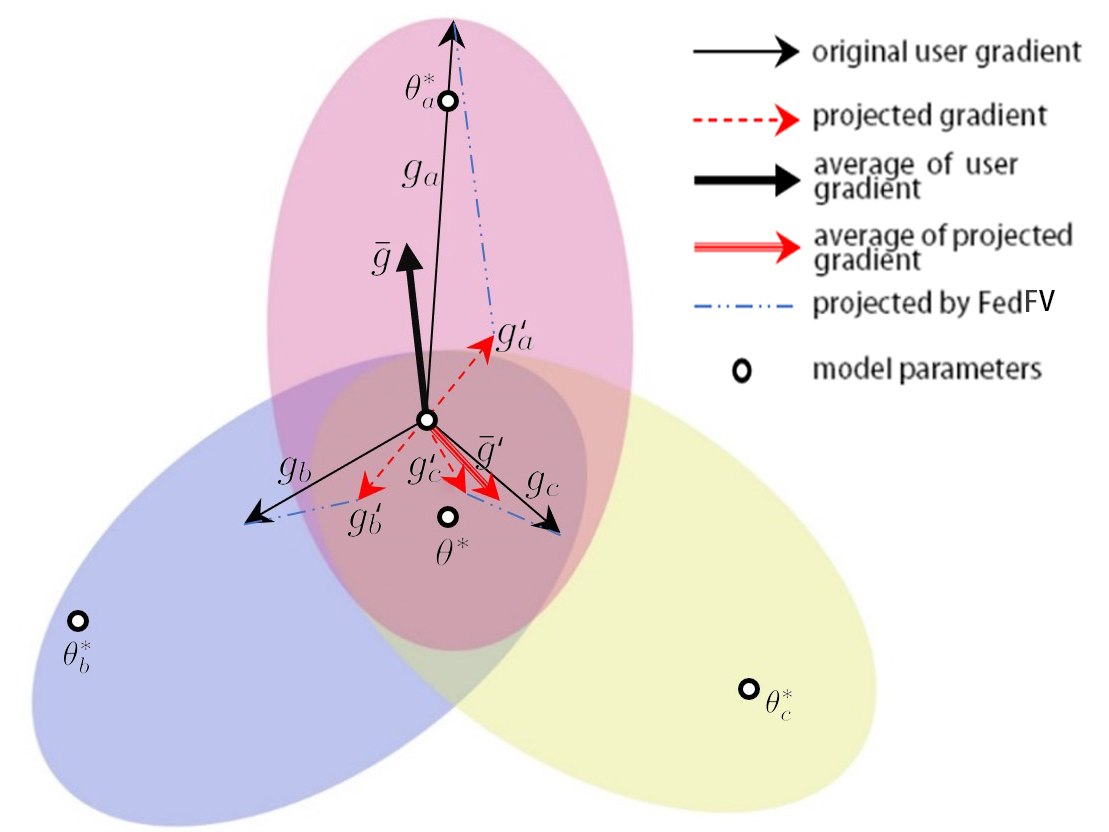}
    \caption{When there are conflicting gradients ($g_a\cdot g_b\textless 0,g_a\cdot g_c\textless 0,g_b\cdot g_c\textless 0$) and large differences in the gradient magnitudes, $||g_a||\textgreater||g_b|| \textgreater||g_c||$, the average of the original gradients $\bar{\theta}$ will be dominated by $g_a$, which will be far away from the global optimal $\theta^*$, resulting in unfairness to the dominated clients $b$ and $c$. We argue that if projecting the gradients to mitigate the conflicts before averaging, the update direction $\bar{g}'$ will be corrected, which is closer to the optimal and fairer.} \label{fig1}
    \end{figure}
    
    Different from these studies, we observe that \emph{conflicting} gradients with large differences in the magnitudes might fatally bring unfairness in FL. Since  gradients are calculated locally by the clients, an update direction of some clients may hinder the model performance on other clients. The gradient conflicts seem not detrimental on their own since it has been shown that a simple average of gradients is valid to optimize the global objective function \cite{ontheconvergence}. However, when the conflicting gradients have large differences in their magnitudes, model accuracy can suffer significant reductions for some of the clients, as shown in Figure 1. 
    
    We highlight three characteristics of FL that make it critical in ensuring fairness for the clients.

    \begin{itemize}
    \item \textbf{Heterogeneous and imbalanced data}. Different clients have their own data characteristics and hence potentially different data distributions. When the data is non-IID distributed, the divergence among the gradients of difference clients gets even larger \cite{fedwithnoniid}. These contribute conflicting gradients in FL. Further, different clients may also have different dataset sizes. If the same local batch size is used across the clients, a client with more data will take more steps to train the local model, thus leading to large differences in gradient magnitudes. As a result, the selected users in a communication round can largely conflict with each other, leading to an unfair update to some of them -- we call such conflicts \emph{internal conflicts}.
    \item \textbf{Party selection}. In each communication round, the server only randomly selects a subset of clients to train the global model with their local datasets. There is no guarantee that the data distribution in each round is representative of the real population distribution. The data distribution can be highly imbalanced when the server repeatedly chooses a particular type of clients, e.g., due to sampling strategies based on computing power, latency, battery etc., the average of gradients may become dominated by the repeatedly selected clients. Therefore, apart from the internal conflicts, we also identify \emph{external conflicts} between those selected and those ignored by the server.
    \item \textbf{Client dropping out}. Another source of \emph{external conflicts} comes from network interruption. When a client drops out due to network interruption (or other hardware issues), its magnitude is considered as 0. An update in the global model may also be unfair for the client. 
    \end{itemize}
    
    To address the issues above, we propose \textbf{Fed}erated \textbf{F}air a\textbf{V}eraging (\textbf{FedFV}) algorithm to mitigate the conflicts among clients before averaging their gradients. We first use the cosine similarity to detect gradient conflicts and then iteratively eliminate such conflicts by modifying both the direction and magnitude of the gradients.  We further show how well FedFV can mitigate the conflicts and how it converges to either a pareto stationary solution or the optimal on convex problems. Extensive experiments on a suite of federated datasets confirm that FedFV compares favorably against state-of-the-art methods in terms of fairness, accuracy and efficiency.
    
    The contributions of this work are summarized as follows:
    \begin{itemize}
    \item We identify two-fold gradients conflicts (i.e., internal conflict and external conflict) that are major causes of unfairness in FL.
    \item We propose FedFV that complements existing fair FL systems and prove its ability to mitigate two gradient conflicts and converge to Pareto stationary solutions.
    \item We perform extensive experiments on a suite of federated datasets to validate the competitiveness of FedFV against state-of-the-art methods in terms of fairness, accuracy and efficiency.
\end{itemize}
\section{Preliminaries}
    FL aims to find a shared parameter $\theta$ that 
    minimizes the weighted average loss of all clients \cite{fedavg}:
    \begin{align}
        \min_\theta F(\theta) =& \sum_{k=1}^K p_k F_k(\theta)
    \end{align}%
    where $F_k(\theta)$ denotes the local objective of the $k$th client 
    with weight $p_k$, $p_k\ge0$ and $\sum_{k=1}^K p_k=1$. 
    The local objective is usually defined by an empirical risk on a local 
    dataset, i.e., $F_k(\theta)=\frac 1{n_k}\sum_{\xi_i\in D_k} l(\theta,\xi_i)$, 
    where $D_k$ denotes the local dataset of the $k$th client, and $n_k$ is the 
    size of $D_k$. To optimize this objective, McMahan \emph{et al.} \shortcite{fedavg} propose \emph{FedAvg}, an iterative algorithm where the server randomly samples a 
    subset $S_t$ of $m$ clients, $0\textless m\le K$, to train the global model with their own datasets 
    and aggregates local updates $G_t=\{g_1^t,g_2^t,...,g_m^t\}$ into an average $g^t=\sum_{k=1}^m p_k g_k^t$ with weight $p_k=\frac{n_k}{\sum_{j=1}^{m}n_j}$ in the $t$th iteration. 
    FedAvg has been shown to be effective in minimizing the objective with low communication 
    costs while preserving privacy. However, it may lead to an unfair result 
    where accuracy distribution is imbalanced among different clients 
    \cite{qfedavg,afl}.
    
    In this paper, we aim to minimize Equation (1) while achieving a fair accuracy distribution among the clients. 


\section{Proposed Approach}
We first present our approach, FedFV, for easing the negative impact of conflicting gradients with largely different magnitudes on FL fairness. We then give a theoretical analysis on how well FedFV achieves fairness and its convergence. We start by defining gradient conflicts following Yu et al. \shortcite{g_surgery}.
\newtheorem{definition}{Definition}
\begin{definition}
Client $i$'s gradient $g_i\in R^d$ conflicts with client $j$'s gradient $g_j\in R^d$ iff $g_i\cdot g_j\textless 0$.
\end{definition}
\subsection{FedFV}
FedFV aims to mitigate the conflicts among clients before averaging their gradients. In the $t$th communication round of FedFV, after receiving the selected clients' updates $G_t=\{g_1^t,g_2^t,...,g_m^t \}$ and training losses $L_t=\{l_1^t,l_2^t,...,l_m^t\}$, the server updates the gradient history $GH$ to keep a trace of the latest gradient of each client. Then, the server sorts gradients in $G_t$ in ascending order of their corresponding clients' losses to obtain $PO_t$, which provides the order of each gradient to be used as a projection target. Finally, FedFV mitigates the internal conflicts and the external conflicts sequentially. Algorithm \ref{alg:1} summarizes the steps of FedFV. 
    \begin{algorithm}[tb]
    \caption{FedFV}
    \label{alg:1}
    \textbf{Input}:$T,m,\alpha,\tau,\eta,\theta^0,p_k,k=1,...,K,$\\
    \begin{algorithmic}[1]
    \STATE Initialize $\theta_0$ and gradient history $GH=[]$.
    \FOR{$t = 0,1,...,T-1$}
    \STATE Server samples a subset $S_t$ of $m$ clients with the prob. $p_k$ and sends the model $\theta^t$ to them.
    \STATE Server receives the updates $g_k^t$ and the training loss $l_k^t$ from each client $k\in S_t$, where $g_k^t\leftarrow\theta^t-\theta_k^t$ and $\theta_k^t$ is updated by client $k$. Then server updates the trace $GH=[g_1^{t_1},g_2^{t_2},...,g_K^{t_K}]$ of the latest updates of all clients.
    \STATE Server sorts the clients' updates into a projecting order list $PO_t=[g_{i_1}^t,...,g_{i_m}^t]$ based on their losses, where $l_{i_j}^t\le l_{i_{j+1}}^t$.
    \STATE $g^t\leftarrow $MitigateInternalConflict$(PO_t,G_t,\alpha)$
    \IF{$t\ge \tau$}
    \STATE $g^t\leftarrow $MitigateExternalConflict$(g^t, GH, \tau)$
    \ENDIF
    \STATE $g^t=g^t/||g^t||*||\frac 1m\sum_k^m g_k^t||$
    \STATE Server updates the model $\theta^{t+1}\leftarrow \theta^t-g^t$
    \ENDFOR
    \end{algorithmic}
    \end{algorithm}

\subsection{Mitigating Internal Conflicts}
We first handle the fairness among the selected clients in each communication round. 
\begin{definition}
In the $t$th communication round, there are internal conflicts if there is at least a pair of client gradients $<g_i^t,g_j^t>$ such that $g_i^t$ conflicts with $g_j^t$, where $g_i^t,g_j^t \in G_t=\{g_1^t,g_2^t,...,g_m^t\}$ and $G^t$ is the selected clients' updates. 
\end{definition}
Internal conflicts account for the unfairness among selected clients. Consider learning a binary classifier. if clients with data of one class are in the majority in a communication round and there exist conflicts between gradients of these two classes, the global model will be updated to favor clients of the majority class, sacrificing the accuracy on clients of the other class. Further, even when the proportion of clients of the two classes is balanced, the magnitudes of gradients may still largely differ due to different dataset sizes on different clients. 

To address the internal conflicts, FedFV iteratively projects a client's gradient onto the normal plane of another client with a conflicting gradient. Instead of projecting in a random order, we design a loss-based order to decide which gradient should be the projecting target of other gradients based on Theorem \ref{th-1}. We show that the later a client gradient is used as the projection target, the fewer conflicts it will have with the final average gradient computed by FedFV. Algorithm \ref{alg:2} summarizes our steps to mitigate internal conflicts. 
\begin{theorem}\label{th-1}
Suppose there is a set of gradients $G=\{g_1,g_2,...,g_m\}$ where $g_i$ always conflicts with $g_j^{(t_j)}$ before projecting $g_j^{(t_j)}$ to $g_i$'s normal plane, and $g_i^{(t_i)}$ is obtained by projecting $g_i$ to the normal planes of other gradients in $G$ for $t_i$ times. Assuming that $|\cos{<g_i^{(t_i)},g_j^{(t_j)}>}|\le \epsilon, 0\textless\epsilon\le 1 $, for each $g_i \in G$, as long as we iteratively project $g_i$ onto $g_k$'s normal plane (skipping $g_i$ it self) in the ascending order of $k$ where $k=1,2,...,m$, the larger $k$ is, the smaller the upper bound of conflicts between the average $\bar{g}'=\frac 1m\sum_{i=1}^m g_i^{(m)}$ and $g_k$ is.
\end{theorem}
\begin{proof}
See Appendix A.1.
\end{proof}



Since projecting gradients favors for the gradient that serves as the projecting target later, we put the clients with larger training losses at the end of the projecting order list $PO_t$ in the $t$th communication round to improve the model performance on the less well trained data. We allow $\alpha m$ clients to keep their original gradients to further enhance fairness. We will detail $\alpha$ in Section \ref{sec:3.3}.
    
    \begin{algorithm}[tb]
    \caption{MitigateInternalConflict}
    \label{alg:2}
    \textbf{Input}: $PO_t,G_t,\alpha$
    \begin{algorithmic}[1]
    \STATE Server selects a set $S_t^{\alpha}$ of the top $\alpha$ clients in $PO_t$.
    \STATE Server sets $g_k^{PC}\leftarrow g_k^t$.
    \FOR{each client $k\in S_t^{\alpha}$ in parallel}
    \FOR{each $g_{i_j}^t\in PO_t,j=1,...,m$}
        \IF{$g_k^{PC}\cdot g_{i_j}^t\textless 0$ and $k\ne i_j$}
        \STATE Server computes $g_k^{PC}\leftarrow g_k^{PC}-\frac{(g_{i_j}^t)\cdot g_k^{PC}}{||g_{i_j}^t||^2}g_{i_j}^t$.
        \ENDIF
    \ENDFOR
    \ENDFOR
    \STATE $g^t\leftarrow\frac{1}{m}\sum_{k=1}^m g_k^{PC}$
    \RETURN $g^t$
    \end{algorithmic}
    \end{algorithm}

\subsection{Mitigate External Conflicts}
Due to party selection and client dropping out, there is a sample bias during each communication round in FL \cite{advance}. A client that is not selected in the $t$th round can suffer a risk of being forgotten by the model whenever the combined update $g^t$ conflicts with its imaginary gradient $g_{img}^t$. We can consider such clients to have a weight of zero in each round. However, we cannot directly detect conflicts for such clients because we have no access to their real gradients. To address this issue, we estimate their real gradients according to their recent gradients, and we call such gradient conflicts external conflicts:
\begin{definition}
In the $t$th communication round, there are external conflicts if there is a client $h\notin S_t$ whose latest gradient $g_h^{t-k}$ conflicts with the combined update $g^t$, where $0\textless k\textless \tau, 0\textless \tau\textless t$. 
\end{definition}
An external conflict denotes the conflict between the assumed gradient of a client that has not been selected and the combined update $g^t$. We add extra steps to prevent the model from forgetting data of clients beyond current selection. We also iteratively project the update $g^t$ onto the normal plane of the average of conflicting gradients for each previous round in a time-based order. The closer the round is, the later we make it the target to mitigate more conflicts between the update $g^t$ and more recent conflicting gradients according to Theorem \ref{th-1}. We scale the length of the update to $||\frac 1m\sum_k^m g_k^t||$ in the end because the length of all gradients can be enlarged by the projection. Algorithm \ref{alg:3} summarizes our steps to mitigate external conflicts. 
    \begin{algorithm}[tb]
    \caption{MitigateExternalConflict}
    \label{alg:3}
    \textbf{Input}: $g^t, GH, \tau$
    \begin{algorithmic}[1]
    \FOR{round $t-i, i=\tau,\tau-1,...,1$}
        \STATE $g_{con}\leftarrow 0$
        \FOR{each client $k=1,2,...,K$}
            \IF{$t_k=t-i$}
                \IF{$g^t\cdot g_{k}^{t_k}\textless 0$}
                \STATE $g_{con}\leftarrow g_{con}+g_{k}^{t_k}$
                \ENDIF
            \ENDIF
        \ENDFOR
        \IF{$g^t\cdot g_{con}\textless 0$}
        \STATE Server computes $g^t\leftarrow g^t-\frac{g^t\cdot g_{con}}{||g_{con}||^2}g_{con}$.
        \ENDIF
    \ENDFOR
    \RETURN $g^t$
    \end{algorithmic}
    \end{algorithm}
\subsection{Analysis}\label{sec:3.3}
We analyze the capability of FedFV to mitigate gradient conflicts and its convergence. We show how FedFV can find a pareto stationary solution on convex problems.
\begin{theorem}\label{th-2}
Suppose that there is a set of gradients $G=\{g_1,g_2,...,g_m\}$ where $g_i$ always conflicts with $g_j^{(t_j)}$ before projecting $g_j^{(t_j)}$ to $g_i$'s normal plane and $g_i^{(t_i)}$ is obtained by projecting $g_i$ to the normal planes of different gradients in $G$ for $t_i$ times. If $\epsilon_1 \le|\cos{<g_i^{(t_i)},g_j^{(t_j)}>}|\le \epsilon_2, 0\textless\epsilon_1\le\epsilon_2\le 1 $, then as long as we iteratively project $g_i$ onto $g_k$'s normal plane (skipping $g_i$ it self) in the ascending order of $k$ where $k=1,2,...,m$, the 
maximum value of $|g_k\cdot \bar{g}'|$ is bounded by $\frac{m-1}{m}(\max_{i}{||g_i||})^2 f(m,k,\epsilon_1,\epsilon_2)$, where $f(m,k,\epsilon_1,\epsilon_2)=\frac{\epsilon_2^2(1-\epsilon_1^2)^{\frac 12}(1-(1-\epsilon_1^2)^{\frac {m-k}2})}{1-(1-\epsilon_1^2)^{\frac 12}}$.
\end{theorem}
\begin{proof}
See Appendix A.2.
\end{proof}
According to Theorem \ref{th-2}, it is possible for FedFV to bound the maximum conflict for any gradient by choosing $k$ to let $f(m,k,\epsilon_1,\epsilon_2)\textless \epsilon_2$ with any given $m,m\ge 2$, since the possible maximum value of conflicts between $g_k$ and the original average gradient is $\max_k{|g_k\cdot \bar{g}|}\le \frac{m-1}{m}\epsilon_2(\max_k{||g_k||})^2$. 

In practice, we mitigate all the conflicts that have been detected
to limit the upper bound of gradient conflicts of clients. 
We allow $\alpha m$ clients with large training losses to keep their original gradients to further enhance fairness. When $\alpha=1$, all clients keep their original gradients so that FedFV covers FedAvg, and when $\alpha=0$, all clients are enforced to mitigate conflicts with others. Parameter $\alpha$ thus controls the degree of mitigating conflicts and can be applied to seek for a balance.

Next, we show that FedFV can find a pareto stationary point according Theorems 3 and 4.
\begin{theorem}\label{th-3}
Assume that there are only two types of users whose objective functions are $F_1(\theta)$ and $F_2(\theta)$, and each objective function is differentiable, $L$-smooth and convex. For the average objective $F(\theta)=\frac 12\sum F_i(\theta)$, FedFV with stepsize $\eta\le\frac 1L$ will converge to either 1) a pareto stationary point, 2) or the optimal $\theta^*$.
\end{theorem}
\begin{proof}
See Appendix A.3.
\end{proof}
\begin{theorem}\label{th-4}
Assume that there are $m$ objective functions $F_1(\theta)$, $F_2(\theta)$,\ldots,  $F_m(\theta)$, and each objective function is differentiable, $L$-smooth and convex. If $\cos{<\bar{g},\bar{g}'>}\ge \frac 12$ and $||\bar{g}||\ge||\bar{g}'||$ where $\bar{g}$ is the true gradient for the average objective $F(\theta)=\frac 1m\sum F_i(\theta)$, then FedFV with step size $\eta\le\frac 1L$ will converge to either 1) a pareto stationary point, 2) or the optimal $\theta^*$.
\end{theorem}
\begin{proof}
See Appendix A.3.
\end{proof}

\section{Related Work}
\subsection{Fairness in FL}
Federated Learning (FL) is first proposed by McMahan et al \shortcite{fedavg} to collaboratively train a global model distributedly while preserving data privacy\cite{advance,fedprox,fedavg}. A number of studies focus on collaborative fairness where the server allocates different models to clients according to their contribution \cite{collaborative,fairpp,building}. A few other studies address the fairness of a uniform accuracy distribution across devices \cite{ucbcs,fedmgda,qfedavg,fedmultitaskcc,afl}. Mohri \emph{et al.} \shortcite{afl} and Li \emph{et al.} \shortcite{qfedavg} propose different federated objectives AFL and q-FFL to further improve fairness. Hu \emph{et al.} \shortcite{fedmgda} observe the competing relation between fairness and robustness to inflating loss attacks. Abay \emph{et al.} \shortcite{mitigate} analyze the potential causes of bias in FL which leads to unfairness, and they also point out the negative impact of sample bias due to the party selection. Cho \emph{et al.} \shortcite{ucbcs} also show that client selection has an impact on fairness, and they propose a bandit-based communication-efficient client selection strategy to overcome biasness. Huang \emph{et al.} \shortcite{fedfa} reweight clients according to their accuracy and numbers of times of being selected to achieve fairness. They use double momentum to accelerate the convergence. Different from these, we identify the conflicts among clients to be a potential cause of unfairness in FL. We mitigate such conflicts by computing a fairer average of gradients to achieve fair model performance across devices.
\subsection{Gradient Projection}
Gradient projection has been well studied in continual learning to mitigate the adversary impact of gradient updates to previously learned tasks \cite{elifelong,orthogonal,learninglongtermrem,gem}. Lopez-Paz \emph{et al.} \cite{gem} project gradient by solving a quadratic programming problem. Chaudhry \emph{et al.}  \shortcite{elifelong} project gradient onto the normal plane of the average gradient of previous tasks. Farajtabar \emph{et al.} \shortcite{orthogonal} project the current task gradients onto the orthonormal set of previous task gradients. Yu \emph{et al.} \shortcite{g_surgery} focus on adversary influence between task gradients when simultaneously learning multiple tasks. They iteratively project each task gradient onto the normal plane of conflicting gradients, which motivates our solution in this paper. To the best of our knowledge, we are the first to take the adversary gradient interference into consideration in FL. Specifically, our proposed FedFV method can build a connection between fairness and conflicting gradients with large differences in the magnitudes, and we prove its ability to mitigate two gradient conflicts (i.e., internal and external conflicts) and converge to Pareto stationary solutions.

\section{Experiments}

\subsection{Experimental Setup}\label{sec:5.1}
\subsubsection{Datasets and Models} We evaluate FedFV on three public datasets: CIFAR-10 \cite{cifar}, Fashion MNIST \cite{fmnist} and MNIST \cite{mnist}. We follow \cite{fedavg} to create non-I.I.D. datasets. For CIFAR-10, we sort all data records based on their classes, and then split them into 200 shards. We use 100 clients, and each client randomly picks 2 shards without replacement so that each has the same data size. We use 200 clients for MNIST and preprocess MNIST in the same way as CIFAR-10. The local dataset is split into training and testing data with percentages of 80\% and 20\%. For Fashion MNIST, we simply follow the setting in \cite{qfedavg}. We use a feedforward neural network with 2 hidden layers on CIFAR-10 and Fashion MNIST. We use a CNN with 2 convolution layers on MNIST. 
\subsubsection{Baselines} We compare with the classical method FedAvg \cite{fedavg} and FL systems that address fairness in FL, including AFL \cite{afl}, q-FedAvg \cite{qfedavg}, FedFa \cite{fedfa} and FedMGDA+ \cite{fedmgda}. We compare with q-fedavg, FedFa and fedmgda+ on all three datasets and compare with AFL only on Fashion MNIST,  because AFL is only suitable for small networks with dozens of clients. 
\subsubsection{Hyper-parameters}For all experiments, we fix the local epoch $E=1$ and use batchsize $B_{CIFAR-10|MNIST}\in\{full, 64\}, B_{Fashion MNIST}\in\{full,400\}$ to run Stochastic Gradient Descent (SGD) on local datasets with stepsize $\eta\in \{0.01,0.1\}$. We verify the different methods with hyper-parameters as listed in Table \ref{tab:1}. We take the best performance of each method for the comparison.
\begin{table}\small
\centering
\begin{tabular}{p{1.5cm}p{6cm}}
\hline
Method      &  Parameters \\
\hline
AFL         & $\eta_{\lambda}\in\{0.01,0.1,0.5\}$    \\
qFedAvg     & $q\in\{0.1,0.2,1,2,5,15\}$     \\
FedMGDA+    & $\epsilon\in\{0.01,0.05,0.1,0.5,1\}$  \\
FedFA     & $(\alpha,\beta)\in\{(0.5,0.5)\},(\gamma_s,\gamma_c)\in\{(0.5,0.9)\}$\\
FedFV     & $\alpha\in\{0,0.1,0.2,\frac 13,0.5,\frac 23\},\tau\in\{0,1,3,10\}$\\
\hline
\end{tabular}
\caption{Method specific hyper-parameters.}
\label{tab:1}
\end{table}
\subsubsection{Implementation} All our experiments are implemented on a 64g-MEM Ubuntu 16.04.6 server with 40 Intel(R) Xeon(R) CPU E5-2630 v4 @ 2.20GHz and 4 NVidia(R) 2080Ti GPUs. All code is implemented in PyTorch version 1.3.1. Please see \url{https://github.com/WwZzz/easyFL} for full details.
\subsection{Experimental Results}

\subsubsection{Fairness}
We first verify the advantage of FedFV in fairness. In Table \ref{tab:2}, we list the the mean, variance, the worst 5\% and the best 5\% of test accuracy on 100 clients created by splitting CIFAR-10 where $10\%$ of clients are sampled in each communication round. While reaching a higher mean of test accuracy across clients, FedFV also yields the lowest variance among all experiments except for q-FedAvg$|_{q=5.0}$. Note that q-FedAvg sacrifices the best performance of clients from $69.28\%$ to $63.60\%$ and causes a significant reduction of the mean of accuracy from $46.85\%$ to $45.25\%$. When reaching a similar variance to q-FedAvg$|_{q=5.0}$ ($9.59$ versus $9.72$), FedFV improves the mean of accuracy to $50.42\%$. Further, the worst 5\% performance of all the experiments on FedFV is higher than that of any others, which indicates FedFV's effectiveness on protecting the worst performance.

\begin{table}
\centering
\setlength{\tabcolsep}{0.3mm}{
\renewcommand\arraystretch{1.3}
\begin{tabular}{>{\scriptsize}l|>{\footnotesize}l|>{\footnotesize}l|>{\footnotesize}l|>{\footnotesize}l}
\hline
\textrm{\footnotesize Method}    & Ave. & Var. & Worst 5\% & Best 5\% \\
\hline
FedAvg     & 46.85$\pm$0.65 & 12.57$\pm$1.50 & 19.84$\pm$6.55 & 69.28$\pm$1.17\\
\hline
qFedAvg$|_{q=0.1}$ & 47.02$\pm$0.89 & 13.16$\pm$1.84 & 18.72$\pm$6.94 & 70.16$\pm$2.06\\
qFedAvg$|_{q=0.2}$ & 46.91$\pm$0.90 & 13.09$\pm$1.84 & 18.88$\pm$7.00 & 70.16$\pm$2.10\\
qFedAvg$|_{q=1.0}$ & 46.79$\pm$0.73 & 11.72$\pm$1.00 & 22.80$\pm$3.39 & 68.00$\pm$1.60\\
qFedAvg$|_{q=2.0}$ & 46.36$\pm$0.38 & 10.85$\pm$0.76 & 24.64$\pm$2.17 & 66.80$\pm$2.02\\
qFedAvg$|_{q=5.0}$ & 45.25$\pm$0.42 &  \textbf{9.59$\pm$0.36} & 26.56$\pm$1.03 & 63.60$\pm$1.13\\
\hline
FedFa & 46.43$\pm$0.56 & 12.79$\pm$1.54 & 19.28$\pm$6.78 & 69.36$\pm$1.40\\
\hline
FedMGDA+$|_{\epsilon=0.01}$& 45.65$\pm$0.21 & 10.94$\pm$0.87 & 25.12$\pm$2.34 & 67.44$\pm$1.20\\
FedMGDA+$|_{\epsilon=0.05}$& 45.58$\pm$0.21 & 10.98$\pm$0.81 & 25.12$\pm$1.87 & 67.76$\pm$2.27\\
FedMGDA+$|_{\epsilon=0.1}$& 45.52$\pm$0.17 & 11.32$\pm$0.86 & 24.32$\pm$2.24 & 68.48$\pm$2.68\\
FedMGDA+$|_{\epsilon=0.5}$& 45.34$\pm$0.21 & 11.63$\pm$0.69 & 24.00$\pm$1.93 & 68.64$\pm$3.11\\
FedMGDA+$|_{\epsilon=1.0}$& 45.34$\pm$0.22 & 11.64$\pm$0.66 & 24.00$\pm$1.93 & 68.64$\pm$3.11\\
\hline
FedFV$|_{\alpha=0.1,\tau=0}$& 48.57$\pm$0.76 & 10.97$\pm$1.02 & 28.32$\pm$2.01 & 69.76$\pm$2.45\\
FedFV$|_{\alpha=0.2,\tau=0}$& 48.54$\pm$0.64 & 10.56$\pm$0.96 & 28.88$\pm$1.92 & 69.20$\pm$2.01\\
FedFV$|_{\alpha=0.5,\tau=0}$& 48.14$\pm$0.38 & 10.60$\pm$0.76 & 28.72$\pm$1.92 & 68.80$\pm$1.77\\
\hline
FedFV$|_{\alpha=0.1,\tau=1}$& 49.34$\pm$0.56 & 10.74$\pm$0.89 & 28.56$\pm$2.54 & 70.24$\pm$1.49\\
FedFV$|_{\alpha=0.1,\tau=3}$& 50.00$\pm$0.74 & 10.85$\pm$1.14 & 28.24$\pm$3.03 &  \textbf{70.96$\pm$1.00}\\
FedFV$|_{\alpha=0.1,\tau=10}$& \textbf{50.42$\pm$0.55} & 9.70$\pm$0.96 &  \textbf{32.24$\pm$2.10} & 69.68$\pm$2.84\\
\hline
\end{tabular}
}
\caption{The average, the variance, the worst and the best of the test accuracy of all clients on CIFAR-10. All experiments are running over 2000 rounds with full batch size, learning rate $\eta=0.1$ and local epochs $E=1$. The reported results are averaged over 5 runs with different random seeds.}
\label{tab:2}
\end{table}
Table \ref{tab:3} lists the model test accuracy on each client's data and the average and variance of the accuracy on Fashion-MNIST dataset. Different from settings in Table \ref{tab:2}, we use a small network with only 3 clients, and the server selects all the clients to participate the training process in each communication round to eliminate the sample bias, which results in the absence of external conflicts. Therefore, we only use FedFV with $\alpha\in\{0,\frac 13,\frac 23\},\tau=0$ in this case. In our observation, even though FedFV with $\alpha=0$ can still find the solution with the highest mean of test accuracy $80.48\%$, it's a unfair solution where the variance is up to $13.76$. Instead, FedFV with $\alpha=\frac 23$ also finds a solution with high accuracy $80.28\%$, and simultaneously keep the variance in a low level. We notice that AFL$|_{\eta_\lambda=0.5}$ has the lowest variance. However, AFL only works for small works where the number of clients is limited. Despite of this, FedFV$|_{\alpha=\frac 23}$ outperforms  AFL$|_{\eta_\lambda=0.5}$ on all the clients (i.e. shirt, pullover and T-shirt), as well as the mean of accuracy. Further, when compared to other methods except AFL, our method has the lowest variance down to 1.77, which validates the ability of FedFV to find a fair and accurate solution without the presence of sample bias.

From Table \ref{tab:2} and Table \ref{tab:3}, we observe that FedFV is able to ease the negative impact of conflicting gradients with largely different magnitudes, which results in a fairer and accurate model with lower variance.

\begin{table}\scriptsize
\centering
\setlength{\tabcolsep}{0.3mm}{
\renewcommand\arraystretch{1.4}
\begin{tabular}{l|l|l|l|l|l}
\hline
Method      & $shirt$& $pullover$& $T-shirt$ & Ave. & Var. \\
\hline
FedAvg & 64.26$\pm$1.07 & 87.03$\pm$1.40 & 89.97$\pm$1.30 & 80.42$\pm$0.76 & 11.50$\pm$0.95\\
\hline
AFL$|_{\eta_\lambda=0.01}$& 71.09$\pm$0.91 & 81.89$\pm$1.30 & 84.29$\pm$1.23 & 79.09$\pm$0.73 & 5.76$\pm$0.75\\
AFL$|_{\eta_\lambda=0.1}$& 76.34$\pm$0.63 & 79.00$\pm$1.04 & 79.06$\pm$1.21 & 78.13$\pm$0.72 & 1.27$\pm$0.66\\
AFL$|_{\eta_\lambda=0.5}$& 76.57$\pm$0.58 & 78.77$\pm$0.99 & 79.09$\pm$1.15 & 78.14$\pm$0.71 & \textbf{1.12$\pm$0.61}\\
\hline
qFedAvg$|_{q=5}$& 71.29$\pm$0.85 & 81.46$\pm$1.08 & 82.86$\pm$0.97 & 78.53$\pm$0.64 & 5.16$\pm$0.69\\
qFedAvg$|_{q=15}$& 77.09$\pm$0.88 & 75.40$\pm$2.07 & 60.69$\pm$1.79 & 71.06$\pm$0.81 & 7.46$\pm$0.88\\
\hline
FedFa & 63.80$\pm$1.10 & 86.86$\pm$1.28 & 89.97$\pm$1.30 & 80.21$\pm$0.73 & 11.68$\pm$0.93\\
\hline
FedMGDA+$|_{\epsilon=0.05}$& 44.63$\pm$1.63 & 57.77$\pm$5.72 & \textbf{99.09$\pm$0.23} & 67.16$\pm$1.65 & 23.39$\pm$0.57\\
FedMGDA+$|_{\epsilon=0.1}$& 72.26$\pm$3.12 & 79.71$\pm$4.47 & 86.03$\pm$4.27 & 79.33$\pm$1.11 & 6.45$\pm$2.21\\
FedMGDA+$|_{\epsilon=1.0}$& 72.46$\pm$3.32 & 79.74$\pm$4.48 & 85.66$\pm$4.88 & 79.29$\pm$1.14 & 6.42$\pm$2.23\\
\hline
FedFV$|_{\alpha=0.0,\tau=0}$& 61.06$\pm$1.21 & \textbf{89.31$\pm$1.38} & 91.06$\pm$1.29 & \textbf{80.48$\pm$0.68} & 13.76$\pm$1.02\\
FedFV$|_{\alpha=\frac 13,\tau=0}$& 77.09$\pm$1.35 & 80.69$\pm$1.53 & 81.06$\pm$0.53 & 79.61$\pm$0.85 & 1.91$\pm$0.57\\
FedFV$|_{\alpha=\frac 23,\tau=0}$& \textbf{77.91$\pm$0.80} & 81.46$\pm$1.99 & 81.46$\pm$1.51 & 80.28$\pm$1.09 & 1.77$\pm$0.87\\
\hline
\end{tabular}
}
\caption{Test accuracy on the different clothing classes of Fashion MNIST dataset. All experiments are running over 200 rounds with full batch size, learning rate $\eta=0.1$ and local epochs $E=1$. The reported results are averaged over 5 runs with different random seeds.}
\label{tab:3}
\end{table}
\subsubsection{Accuracy and Efficiency}
 \begin{table}
\centering
\setlength{\tabcolsep}{0.1mm}{
\renewcommand\arraystretch{1.2}
\begin{tabular}{l|l|l|l}
\hline
     & FedFV &FedFV$_{Random}$& FedFV$_{Reverse}$ \\
\hline
$Var._{CIFAR10}$&\textbf{13.19$\pm$1.06} & 14.12$\pm$0.55 & 16.28$\pm$0.72 \\
$Var._{FMNIST}$ &\textbf{13.76$\pm$1.02} & 20.14$\pm$0.61 & 22.05$\pm$0.68\\  
\hline
\end{tabular}
}
\caption{The effects of projecting order to fairness. The results are averaged over 5 runs with different random seeds.}
\label{tab:4}
\end{table}


We further show that FedFV outperforms existing works addressing fairness in terms of accuracy and efficiency on CIFAR-10 and MNIST. All the methods are tuned to their best performance. As shown in Figure \ref{fig:2}, when we only mitigate the internal conflicts, which is done by FedFV($\tau=0$), we already converge faster while keeping a lower variance than the others. In addition, when we mitigate both the internal and external conflicts by FedFV($\tau\textgreater0$), there is a substantial improvement in the accuracy and efficiency again.
FedFV outperforms state-of-the-art methods by up to 7.2\% on CIFAR-10 and 78\% on MNIST, which verifies the ability of FedFV to reach a higher accuracy with less communication rounds 
while still keeping a low variance. Thus, we demonstrate that FedFV's advantages in saving communication cost and finding a better generalization. 
\begin{figure}
\centering
\subcaptionbox{CIFAR-10\label{subfig:above}}
    {%
        \includegraphics[width = .5\linewidth]{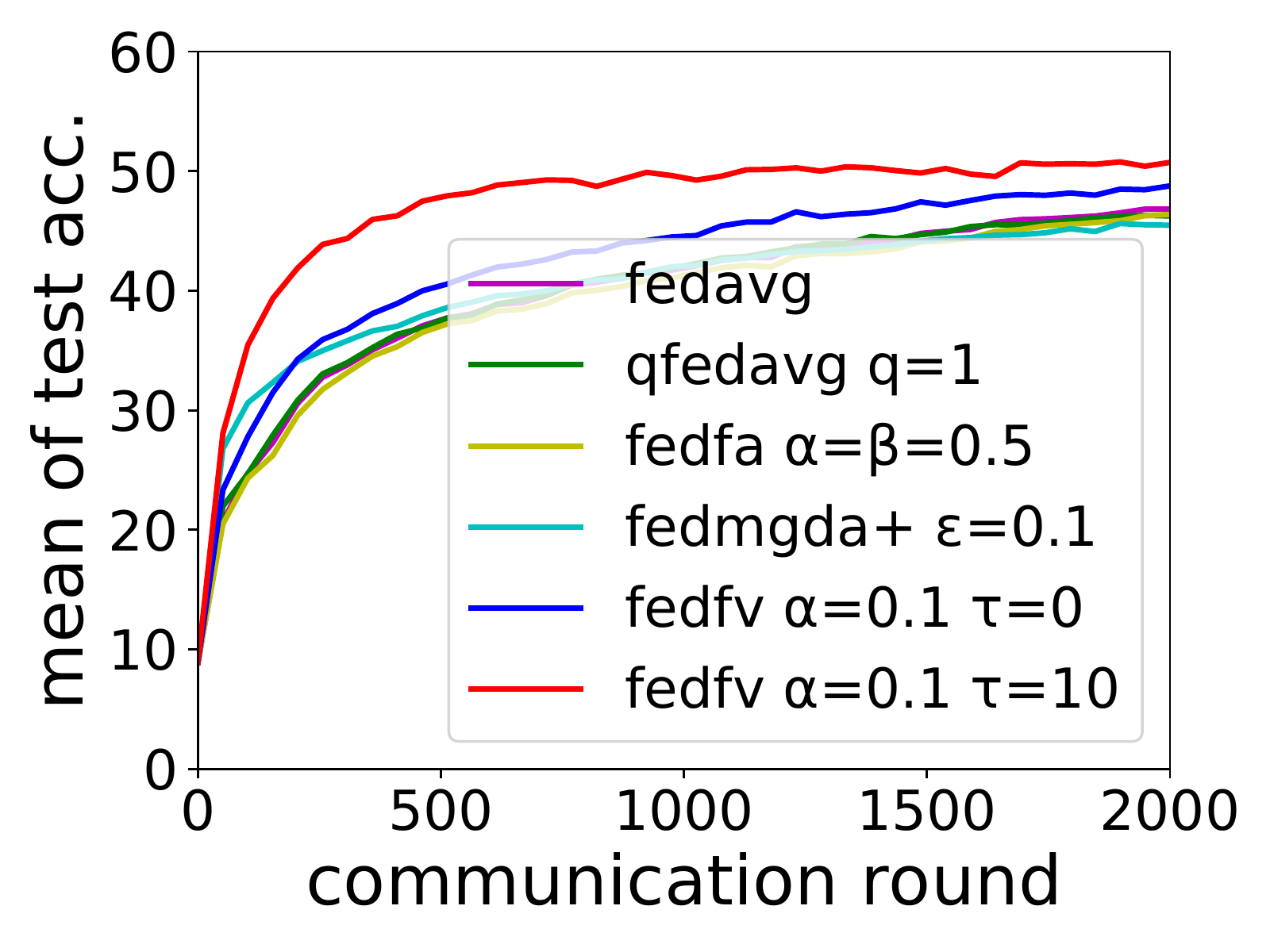}
        \includegraphics[width = .5\linewidth]{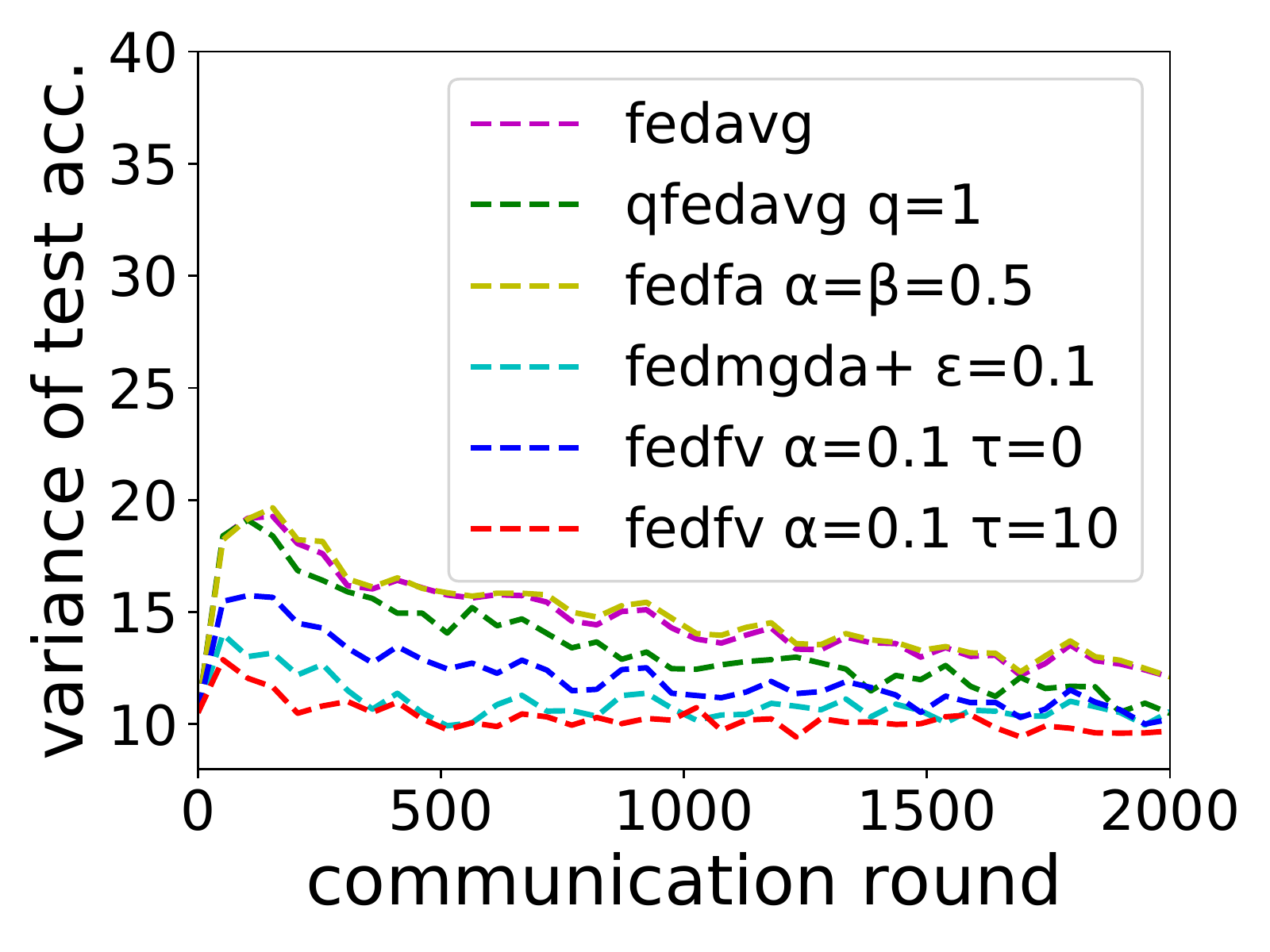}}
\subcaptionbox{MNIST\label{subfig:below}}
    {%
        \includegraphics[width = .5\linewidth]{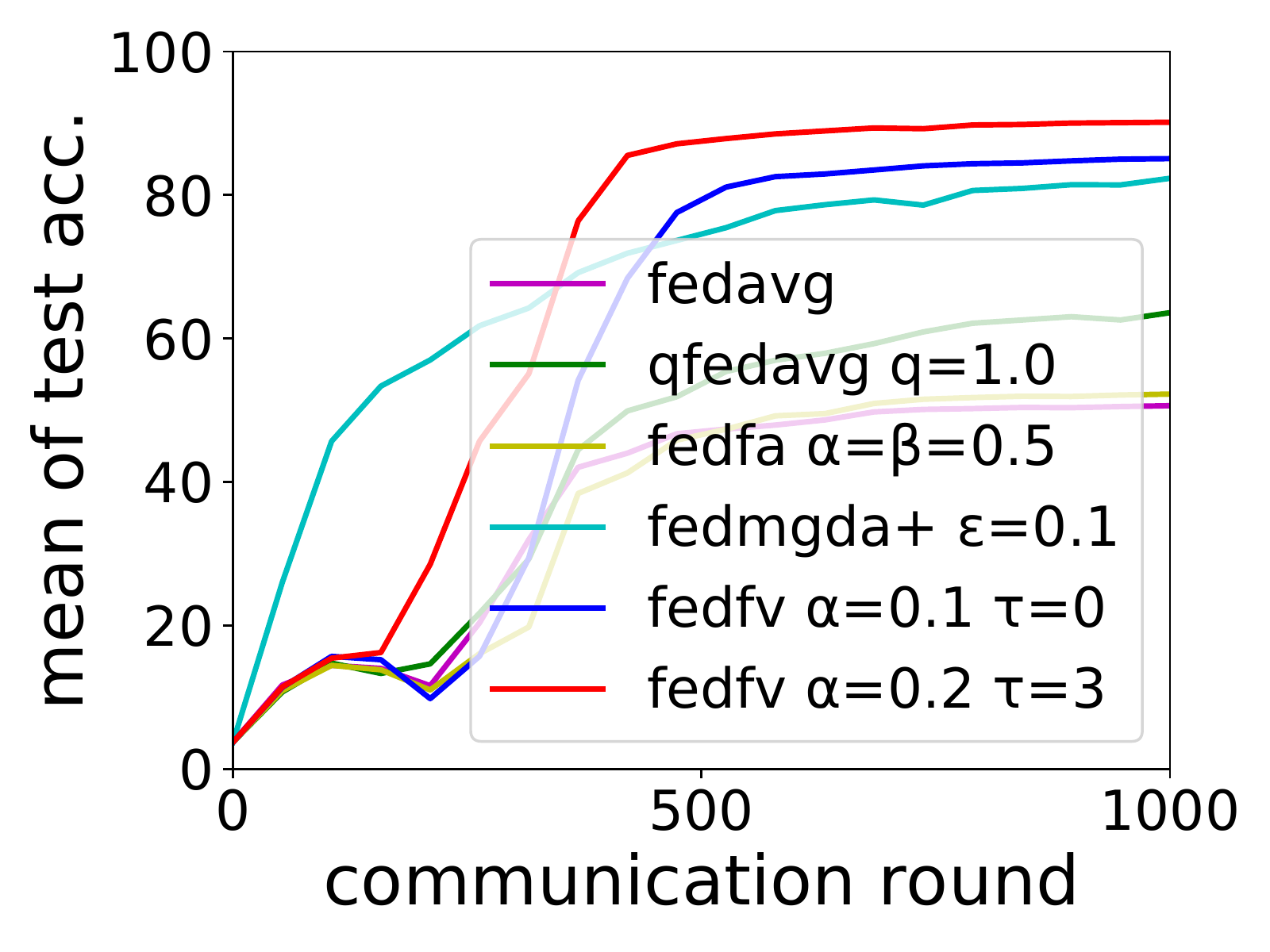}
        \includegraphics[width = .5\linewidth]{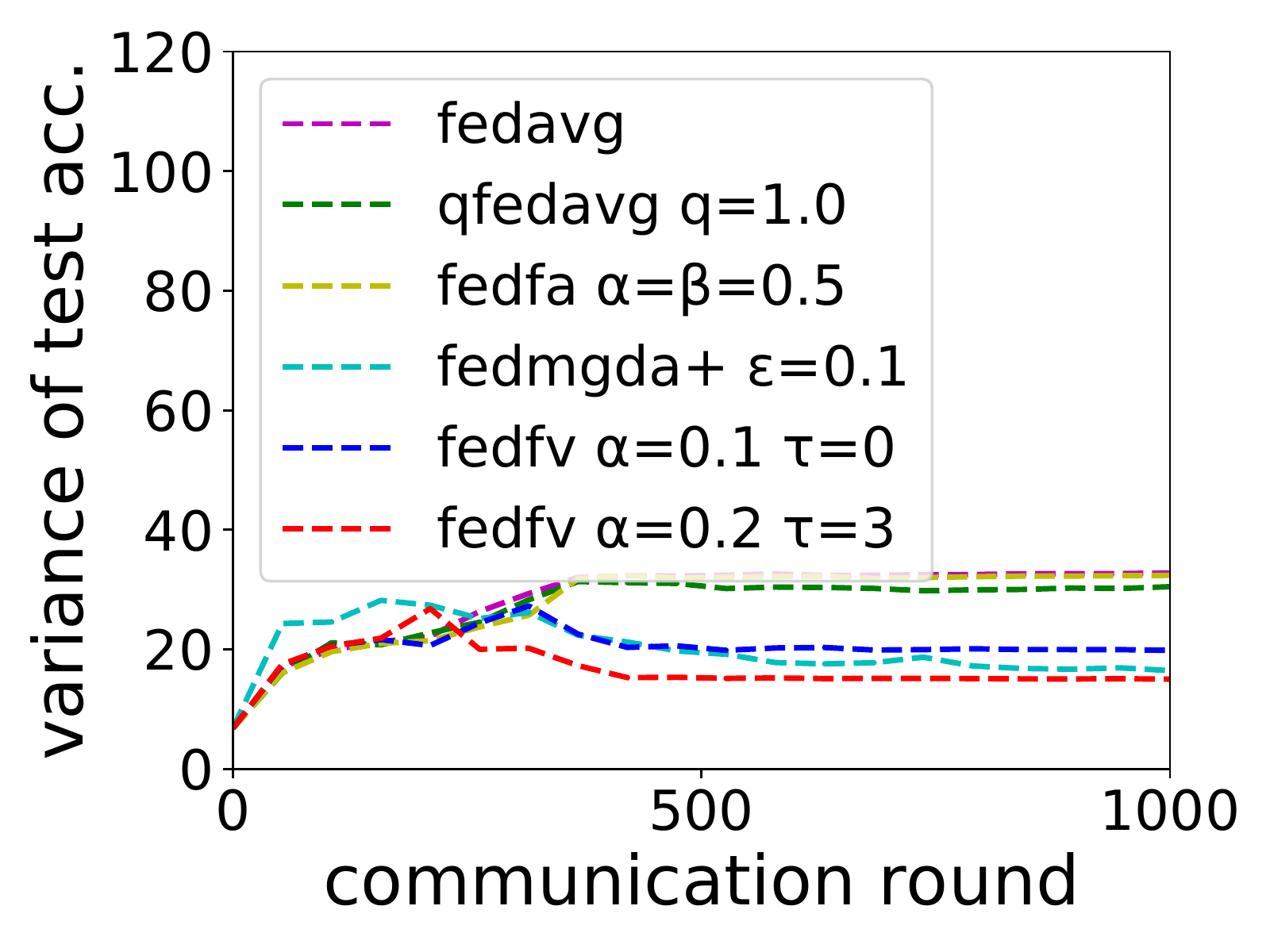}}
\caption{The mean (left) and the variance (right) of test accuracy on all clients on (a) CIFAR-10, and (b) MNIST. The results are averaged over 5 runs with different random seeds.\label{fig:2}}
\end{figure}


\subsubsection{Effects of Projecting Order}

To verify the effectiveness of our projecting order, we compare FedFV with another two cases: 1) projecting gradients in a random order of the projection targets; and 2) projecting gradients in an order of the projection targets that is reverse to FedFV, as shown in Table \ref{tab:4}.

For CIFAR-10, 20\% of clients are sampled in each communication round. For Fashion MNIST, all clients are selected in each communication round. We set $\alpha=0,\tau=0$ for all groups to confirm the effectiveness of the projecting order. If we project gradients to the targets in the loss-based order of FedFV, the variance is the lowest. Projecting gradients to the targets in a random order is also fairer than in an order reverse to FedFV, which indicates that the loss-based order used by FedFV helps improve fairness.

\section{Conclusions and Future Work}

 We identified the issue of conflicting gradients with large differences in the magnitudes, which brings unfairness in FL. To address this issue, we propose the \textbf{Fed}erated \textbf{F}air a\textbf{V}eraging (\textbf{FedFV}) algorithm to mitigate the potential conflicts among clients before averaging their gradients. In addition, we show how well FedFV can mitigate the conflicts and how it converges to either a pareto stationary solution or the optimal on convex problems. Extensive experiments on a suite of federated datasets confirm that FedFV compares favorably against state-of-the-art methods in terms of fairness, accuracy and efficiency. In the future, we plan to build a complete theoretical analysis of FedFV on mitigating complex external conflicts.

\begin{appendix}
\appendix
\newtheorem{thm}{Theorem}
\section{Proof}
\subsection{Proof of Theorem 1}\label{A.1}


\begin{proof}
For each gradient $g_i\in G$, we project $g_i$ onto $g_k$'s normal plane in an increasing order with $k$. Thus, we have update rules
\begin{align*}
\begin{split}
\left \{
\begin{array}{ll}
    g_i^{(0)}=g_i,&k=0\\
    g_i^{(k)}=g_i^{(k-1)}-\frac{g_i^{(k-1)}\cdot g_k}{||g_k||^2}g_k, &k=1,2,...,m,k\ne i\\
    g_i^{(k)}=g_i^{(k-1)},&k=i
\end{array}
\right.
\end{split}
\end{align*}
Therefore, $g_i^{(k)}$ never conflicts with $g_k$, since all potential conflicts have been eliminated by the updating rules. We then focus on how much the final gradient $\bar{g}'$ conflicts with each gradient in $G$. Firstly, the last in the projecting order is $g_m$, and the final gradient $\bar{g}'$ will not conflict with it. Then, we focus on the last but second gradient $g_{m-1}$. Following the update rules, we have
\begin{equation}
\begin{split}
g_i^{(m)}=g_i^{(m-1)}-\frac{g_i^{(m-1)}\cdot g_m}{||g_m||^2}g_m\\
%
 \end{split}
\end{equation}
Let $\phi_{i,j}^{(k)}$ denote the angle between $g_i^{(k)}$ and $g_j$, and $\phi_{i,j}=\phi_{i,j}^{(0)}$, then the average of projected gradients is: 
\begin{equation}
\begin{split}
&\bar g'=\frac 1m\sum_i^m g_i^{(m)}\\
      &=\frac 1m(\sum_{i\ne m}^m (g_i^{(m-1)}-\frac{g_i^{(m-1)}\cdot g_m}{||g_m||^2}g_m)+g_m^{(m-1)})\\
      &=\frac 1m\sum_i^m g_i^{(m-1)}-\frac 1m\sum_{i\ne m}^m ||g_i^{(m-1)}||\cos{\phi_{i,m}^{(m-1)}}\frac{g_m}{||g_m||}
%
\end{split}
\end{equation}
Since $g_{m-1}\cdot \sum_i^m g_i^{(m-1)}\ge0$, 
\begin{equation}
\begin{split}
&g_{m-1}\cdot \bar g'\ge g_{m-1}\cdot \frac {-1}m\sum_{i\ne m}^m ||g_i^{(m-1)}||\cos{\phi_{i,m}^{(m-1)}}\frac{g_m}{||g_m||}\\
&=-\frac 1m \sum_{i\ne m}^m ||g_i^{(m-1)}||||g_{m-1}||\cos{\phi_{i,m}^{(m-1)}}\cos{\phi_{m-1,m}}\\
&\ge -\frac{\epsilon^2}{m}||g_{m-1}||\sum_{i\ne m}^m ||g_i^{(m-1)}||
\end{split}
\end{equation}
Similarly, we can compute the conflicts between any gradient $g_k\in G$ and $\bar g'$ by removing the $g_i^{(k)}$ from $\bar g'$
\begin{equation}
\begin{split}
g_k\cdot \bar g'&\ge g_{k}\cdot\sum_{j=k}^{m-1} (\frac {-1}m\sum_{i\ne j+1}^m ||g_i^{(j)}||\cos{\phi_{i,j+1}^{(j)}}\frac{g_{j+1}}{||g_{j+1}||})\\
& = -\frac {||g_k||}m \sum_{j=k}^{m-1} \sum_{i\ne j+1}^m||g_i^{(j)}||\cos{\phi_{i,j+1}^{(j)}}\cos{\phi_{k,j+1}}\\
&\ge  -\frac{\epsilon^2}{m}||g_k|| \sum_{j=k}^{m-1} \sum_{i\ne j+1}^m ||g_i^{(j)}||\\
|e_k\cdot \bar g'|&=|\frac {g_k}{||g_k||}\cdot \bar g'|\le \frac{\epsilon^2}{m}\sum_{j=k}^{m-1} \sum_{i\ne j+1}^m ||g_i^{(j)}||
\end{split}
\end{equation}
Therefore, the later $g_k$ serves as the projecting target of others, the smaller the upper bound of conflicts between $\bar g'$ and $g_k$ is, since $\sum_{j=k}^{m-1} \sum_{i\ne j+1}^m ||g_i^{(j)}||\le\sum_{j=k-1}^{m-1} \sum_{i\ne j+1}^m ||g_i^{(j)}||$.
\end{proof}
\subsection{Proof of Theorem 2}\label{A.2}

\begin{proof}
According to (5), the upper bound of the conflict between any client gradient $g_k$ and the final average gradient $\bar{g}'$ can be expressed as

\begin{equation}
\begin{split}
|g_k\cdot \bar g'|\le \frac{\epsilon_2^2}{m}||g_k||\sum_{j=k}^{m-1} \sum_{i\ne j+1}^m ||g_i^{(j)}||
\end{split}
\end{equation}
With the update rules of FedFV, we can infer that
\begin{equation}
\begin{split}
||g_k^{(i)}||^2&=||g_k^{(i-1)}-||g_k^{(i-1)}||\cos{\phi_{k,i}^{(i-1)}}\frac{g_i}{||g_i||}||^2\\
    &=||g_k^{(i-1)}||^2-2||g_k^{(i-1)}||^2\cos^2{\phi_{k,i}^{(i-1)}}+\\
    &||g_k^{(i-1)}||^2\cos^2{\phi_{k,i}^{(i-1)}}\\
    &=(1-\cos^2{\phi_{k,i}^{(i-1)}})||g_k^{(i-1)}||^2\\
    &\le (1-\epsilon_1^2)||g_k^{(i-1)}||^2
\end{split}
\end{equation}
Therefore, the maximum value of gradient conflict is bounded by
\begin{equation}
\begin{split}
&|g_k\cdot \bar{g}'|\le \frac{\epsilon_2^2}{m}||g_k|| \sum_{j=k}^{m-1}\sum_{i\ne j+1}^m ||g_i^{(j)}||\\
& \le  \frac{\epsilon_2^2}{m} (\max_{i}{||g_i||}) \sum_{j=k}^{m-1}\sum_{i\ne j+1}^m ||g_i^{(0)}||(1-\epsilon_1^2)^{\frac j2}\\
& \le \frac{m-1}{m}\epsilon_2^2(\max_{i}{||g_i||})^2 \sum_{j=k}^{m-1} (1-\epsilon_1^2)^{\frac j2}\\
&= \frac{m-1}{m}(\max_{i}{||g_i||})^2 \frac{\epsilon_2^2(1-\epsilon_1^2)^{\frac 12}(1-(1-\epsilon_1^2)^{\frac {m-k}2})}{1-(1-\epsilon_1^2)^{\frac 12}}\\
&= \frac{m-1}{m}(\max_{i}{||g_i||})^2 f(m,k,\epsilon_1,\epsilon_2)
\end{split}
\end{equation}
\end{proof}

\subsection{Proof of the Convergence}\label{A.3}
\begin{proof}[Proof of Theorem 3]
Let $g_1$ and $g_2$ be the two types of users' updates of parameters in the $t$th communication round.
When $F(\theta)$ is $L$-smooth, we have
\begin{equation}
F(\theta^{t+1})\le F(\theta^t)+\nabla F(\theta)^T(\theta^{t+1}-\theta^t)+\frac 12L||\theta^{t+1}-\theta^{t}||_2^2
\end{equation}
There are two cases: conflicting gradients exist or otherwise.
If there is no conflict between $g_1$ and $g_2$, which indicates $g_1\cdot g_2\ge 0$, FedFV updates as FedAvg does, simply computing an average of the two gradients and adding it to the parameters to obtain the new parameters $\theta^{t+1}=\theta^t-\eta \bar g$, where $g=\frac 12(g_1+g_2)$. Therefore, when using step size $\eta \le \frac 1L$, it will strictly decrease the objective function $F(\theta)$ \cite{ontheconvergence}. However, if $g_1\cdot g_2\textless 0$, FedFV will compute the update as:
\begin{equation}
\theta^{t+1}=\theta^t-\eta \bar{g}'=\theta^t-\eta(g_1+g_2-\frac{g_1\cdot g_2}{||g_1||^2}g_1-\frac{g_1\cdot g_2}{||g_2||^2}g_2)
\end{equation}
Combining with (9), we have:
\begin{equation}
\begin{split}
 &F(\theta^{t+1})\le F(\theta^t)-\eta(g_1+g_2)^T(g_1+g_2-\frac{g_1\cdot g_2}{||g_1||^2}g_1-\\
 &\frac{g_1\cdot g_2}{||g_2||^2}g_2)+\frac L2\eta^2||g_1+g_2-\frac{g_1\cdot g_2}{||g_1||^2}g_1-\frac{g_1\cdot g_2}{||g_2||^2}g_2||_2^2\\
 &=F(\theta^t)-\eta||g_1+g_1||^2+\eta(g_1+g_2)^T(\frac{g_1\cdot g_2}{||g_1||^2}g_1\\
 &+\frac{g_1\cdot g_2}{||g_2||^2}g_2)+\frac L2\eta^2||g_1+g_2-\frac{g_1\cdot g_2}{||g_1||^2}g_1-\frac{g_1\cdot g_2}{||g_2||^2}g_2||_2^2
 \end{split}
\end{equation}
Since $g_1\cdot g_2=||g_1||||g_2||\cos{\phi_{12}}$, where $\phi_{12}$ denotes the angle between $g_1$ and $g_2$, after rearranging the items in this inequality, we have
\begin{equation}
\begin{split}
 F(\theta^{t+1})\le F(\theta^t)-(\eta-\frac L2\eta^2)(1-\cos^2 \phi_{12})(||g_1||^2+\\
 ||g_2||^2)-Lt^2(1-\cos^2 \phi_{12})||g_1||||g_2||\cos \phi_{12}
 \end{split}
\end{equation}
To decrease the objective function, the inequality below should be satisfied
\begin{equation}
\begin{split}
-(\eta-\frac L2\eta^2)(1-\cos^2 \phi_{12})(||g_1||^2+||g_2||^2)-L\eta^2(1-\\
\cos^2 \phi_{12})||g_1||||g_2||\cos \phi_{12} \le 0
 \end{split}
\end{equation}
\begin{equation}
\begin{split}
(\eta-\frac L2\eta^2)(||g_1||^2+||g_2||^2)+\frac L2\eta^22||g_1||||g_2||\cos \phi_{12}\ge 0
 \end{split}
\end{equation}
\begin{equation}
\begin{split}
(\eta-L\eta^2)(||g_1||^2+||g_2||^2)+\frac L2\eta^2(||g_1||^2+||g_2||^2+\\
2||g_1||||g_2||\cos \phi_{12})\ge 0
 \end{split}
\end{equation}
\begin{equation}
\begin{split}
(\eta-L\eta^2)(||g_1||^2+||g_2||^2)+\frac L2\eta^2||g_1+g_2||^2\ge 0
 \end{split}
\end{equation}
With $\eta\le\frac1L$, we have $\eta-L\eta^2\ge0$, which promises the decrease of the objective function. On the other hand, we can infer that
\begin{equation}
\begin{split}
-\eta(1-\frac L2\eta)=\eta(\frac L2\eta-1)\le \eta(\frac 12 -1)=-\frac \eta2
 \end{split}
\end{equation}
Combine (17) with (13), we have:
\begin{equation}
\begin{split}
 &F(\theta^{t+1}) \le  F(\theta^t)-\frac \eta2(1-\cos^2\phi_{12})(||g_1||^2+||g_2||^2)\\
 &-\frac \eta2(1-\cos^2 \phi_{12})2||g_1||||g_2||\cos \phi_{12}\\
 &=F(\theta_{t})-\frac \eta2(1-\cos^2 \phi_{12})(||g_1||^2+||g_2||^2+2||g_1||||g_2||\\
 &\cos \phi_{12})\\
 &=F(\theta_{t})-\frac \eta2(1-\cos^2 \phi_{12})||g_1+g_2||^2\\
 &=F(\theta_{t})-\frac \eta2(1-\cos^2 \phi_{12})||g||^2
 \end{split}
\end{equation}
Therefore, the objective function will always degrade unless 1)$||g||=0$, which indicates it will reach the optimal $\theta^*$, 2)$\cos{\phi_{12}}=-1$, then we can suppose $g_1=-\beta g_2$ where $\beta\textgreater 0$, and we will get
\begin{equation}
\begin{split}
\frac 1{1+\beta}g_1+\frac \beta{1+\beta} g_2=0
 \end{split}
\end{equation}
So $\theta^t$ is a pareto stationary point as defined below:
\begin{definition}
For smooth criteria $F_k(\theta)(1\le k\le K)$, $\theta^0$ is called a Pareto-stationary iff there exists some convex combination of the gradietns {$\nabla F_k(\theta^0)$} that equals zero \cite{mgda}.
\end{definition}
\end{proof}
\begin{proof}[Proof of Theorem 4]
With the assumption that $F(\theta)$ is $L$-smooth, we can get the inequality:
\begin{equation}
F(\theta^{t+1})\le F(\theta^t)+\nabla F(\theta)^T(\theta^{t+1}-\theta^t)+\frac 12L||\theta^{t+1}-\theta^{t}||_2^2.
\end{equation}
Similar to the proof for Theorem 2, if there exist conflicts, then
\begin{equation}
\begin{split}
F(\theta^{t+1})&\le F(\theta^t)-\eta \bar{g} \cdot \bar{g}'+\frac 12L\eta^2||\bar{g}'||^2\\
&\le F(\theta^t)-\frac\eta 2 ||\bar{g}||||\bar{g}'||+\frac L2 \eta^2 ||\bar{g}'||^2\\
&\le F(\theta^t)-\frac\eta 2 ||\bar{g}||||\bar{g}'||+\frac L2 \eta^2 ||\bar{g}||||\bar{g}'||\\
&\le F(\theta^t)+(-\frac\eta 2 +\frac L2 \eta^2 )||\bar{g}||||\bar{g}'||\\
\end{split}
\end{equation}
Since $(-\frac\eta 2 +\frac L2 \eta^2 )\le 0$ with $\eta\le \frac 1L$, the average objective will always degrade if we repeatedly apply the update rules of FedFV unless 1)$||\bar{g}||=0$, which indicates it will finally reach the optimal $\theta^*$, 2) $||\bar{g}'||=0$, which indicates that all pairs of conflicting gradients have a cosine similarity of $-1$, leading to an existence of convex combination $p$ which satisfies $\sum p_i g_i=0$, since we can easily choose a pair of conflicting gradients $g_i$ and $g_j$ as the proof of theorem 3 does and then set the weights of the rest zero.
\end{proof}
\end{appendix}
\clearpage

\bibliographystyle{named}
\end{document}